\documentclass[10pt,twocolumn,letterpaper]{article}

\usepackage{cvpr}
\usepackage{times}
\usepackage{epsfig}
\usepackage{graphicx}
\usepackage{amsmath}
\usepackage{amssymb}
\usepackage{macro}
\usepackage{amsthm}
\usepackage{algorithm}
\usepackage{algorithmic}
\usepackage{url}
\usepackage{multirow}
\usepackage{epstopdf}

\usepackage[bf]{caption}
\usepackage[labelfont=bf]{subcaption}

\newtheorem{proposition}{Proposition}

\graphicspath{{./}{./Fig/}{./fig/}{./Figs/}{./figs/}}

\newcommand{\BY}{{\mathbf{Y}}}
\newcommand{\BX}{{\mathbf{X}}}

\newcommand{\Bone}{{\mathbf{1}}}

\newcommand{\BB}{{\mathbf{B}}}
\newcommand{\Bb}{{\mathbf{b}}}
\newcommand{\BP}{{\mathbf{P}}}

\newcommand{\BW}{{\mathbf{W}}}
\newcommand{\Bw}{{\mathbf{w}}}

\newcommand{\BZ}{{\mathbf{Z}}}

\newcommand{\BH}{{\mathbf{H}}}
\newcommand{\BI}{{\mathbf{I}}}
\newcommand{\Bx}{{\mathbf{x}}}

\newcommand{\Bz}{{\mathbf{z}}}
\newcommand{\Bg}{{\mathbf{g}}}

\renewcommand{\Lambda}{\varLambda}

\newcommand{\st}{{\,\,\mathrm{s.t.\,\,}}}

\newcommand{\trace}{{\mathrm{trace}}}

\newcommand{\sgn}{{\mathrm{sgn}}}

\cvprfinalcopy 

\def\cvprPaperID{168} 

\ifcvprfinal\fi

\usepackage{cite}

\usepackage{eucal,bibspacing}

\begin{document}

\title{Learning  Binary Codes and Binary Weights for Efficient Classification}

\author{Fumin Shen$^{\dagger}$, Yadong Mu$^{\ddagger}$,  Wei Liu$^{\sharp}$, Yang Yang$^{\dagger}$, Heng Tao Shen$^{\dagger\natural}$
\\
$^\dagger$ University of Electronic Science and Technology of China\\
$^\ddagger$ AT\&T Labs Research
$^\sharp$ Didi Research, Beijing China
$^\natural$ The University of Queensland
}

\maketitle
\thispagestyle{empty}

\begin{abstract}

This paper proposes a generic formulation that significantly expedites the training and deployment of image classification models, particularly under the scenarios of many image categories and high feature dimensions. As a defining property, our method represents both the images and learned classifiers using binary hash codes, which are simultaneously learned from the training data. Classifying an image thereby reduces to computing the Hamming distance between the binary codes of the image and classifiers and selecting the class with minimal Hamming distance. Conventionally, compact hash codes are primarily used for accelerating image search. Our work is first of its kind to represent classifiers using binary codes. Specifically, we formulate multi-class image classification as an optimization problem over binary variables. The optimization alternatively proceeds over the binary classifiers and image hash codes. Profiting from the special property of binary codes, we show that the sub-problems can be efficiently solved through either a binary quadratic program (BQP) or linear program. In particular, for attacking the BQP problem, we propose a novel bit-flipping procedure which enjoys high efficacy and local optimality guarantee. Our formulation supports a large family of empirical loss functions and is here instantiated by exponential / hinge losses. Comprehensive evaluations are conducted on several representative image benchmarks. The experiments consistently observe reduced complexities of model training and deployment, without sacrifice of accuracies. 
\end{abstract}

\section{Introduction}

In recent years, large-scale visual recognition problem has attracted tremendous research enthusiasm from both academia and industry owing to the explosive increase of data size and feature dimensionality. Classifying an image into thousands of categories often entails heavy computations by using a conventional classifier, exemplified by $k$ nearest neighbor ($k$-NN) and support vector machines (SVM), on a commodity computer. For the image recognition problem with many categories, the computational and memory overhead primarily stems from the large number of classifiers to be learned.
The complexities can be high at the stages of both training and deploying these classifier. Considering a classification task with $C$ different classes and $D$-dimensional feature representation, even the simplest linear models are comprised of $D \times C$ parameters. As an inspiring example to our work in this paper, the ImageNet dataset~\cite{deng2009imagenet} contains annotated images from 21,841 classes in total. When experimenting with some state-of-the-art visual features (\eg, 4096-dimensional deep neural networks feature), a huge number of 80 million parameters need to be learned and stored, which clearly indicates slow training and low efficacy at the deployment phase. Real-world applications (such as industrial image search engine) often require near-real-time response. The conventional ways of training multi-class image classifiers thus have much space to be improved.

Compact binary hash codes~\cite{LSH99} have demonstrated notable empirical success in facilitating large-scale similarity-based image search, referred to as image hashing in the literature. In a typical setting of supervised learning, the hash codes are optimized to ensure smaller Hamming distances between images of the same semantic kind. In practice, image hashing techniques have been widely utilized owing to its low memory footprint and theoretically-guaranteed scalability to large data.

Though the hashing techniques for image search has been a well-explored research area, its application on large-scale optimization still remains a nascent topic in the fields of machine learning and computer vision. Intuitively, one can harness the hash codes for the image classification task through naive methods such as $k$-NN voting. Both the training and testing images are indexed with identical hashing functions. A new image is categorized by the majority semantic label within the hashing bucket where it is projected into. However, since the hash codes are initially optimized for image search purpose, such a naive scheme does not guarantee high accuracy for image recognition.

The most relevant works to ours are approximately solving non-linear kernel SVM via hashing-based data representation~\cite{li2011hashing,li2013sign,mu2014hash}. These methods first designate a set of hashing functions that transform the original features into binary codes. The original non-linear kernels (\eg, RBF kernel) are theoretically proved to be approximated by the inner product between binary hash bits. Prominent advantages of such a treatment are two-folds: the required hash bits only weakly hinge on the original feature dimensionality, and meanwhile the non-linear optimization problem is converted into a linear alternative. As a major drawback, these works still rely on the regular real-valued based classifiers upon the binary features. Though it enables the direct application of linear solvers, the potential of binary codes is not fully utilized.

Our work is a non-trivial extension of the aforementioned line of research. We further elevate the efficacy of classification models by binarizing both the features and the weights of classifiers. In other words, our goal is to develop a generic multi-class classification framework. The classifier weights and image codes are simultaneously learned. Importantly, both of them are represented by binary hash bits. This way the classification problem is transformed to an equivalent and simpler operation, namely searching the minimal Hamming distance between the query and the $C$ binary weight vectors. This can be extremely fast by using the built-in XOR and popcount operations in modern CPUs. We implement this idea by formulating the problem of minimizing the empirical classification error with purely binary variables.

The major technical contributions of this work are summarized as below:
\begin{enumerate}
\item We define a novel problem by binarizing both classifiers and image features and simultaneously learning them in a unified formulation. The prominent goal is to accelerate large-scale image recognition. Our work represents an unexplored research direction, namely extending hashing techniques from fast image search to the new topic of hashing-accelerated image classification.

\item An efficient solver is proposed for the binary optimization problem. We decouple two groups of binary variables (image codes and binary classifier weights) and adopt an alternating-minimizing style procedure. Particularly, we show that the sub-problems are in the form of either binary quadratic program (BQP) or linear program. An efficient bit-flipping scheme is designed for the BQP sub-problem. Profiting from the special traits of binary variables, we are able to specify the local optimality condition of the BQP.

\item Our formulation supports a large family of empirical loss functions and is here instantiated by exponential / hinge losses. In our quantitative evaluations, both variants are compared with key competing algorithms, particulary a highly-optimized implementation of the SVM algorithm known as LibLinear~\cite{liblinear08}. Our proposed method demonstrates significant superiority in terms of train/test CPU time and the classification accuracy, as briefly depicted by Figure~\ref{fig:acc_sun}.
\end{enumerate}

\begin{figure}
\centering
\begin{subfigure}[b]{0.23\textwidth}
\includegraphics[width=\textwidth]{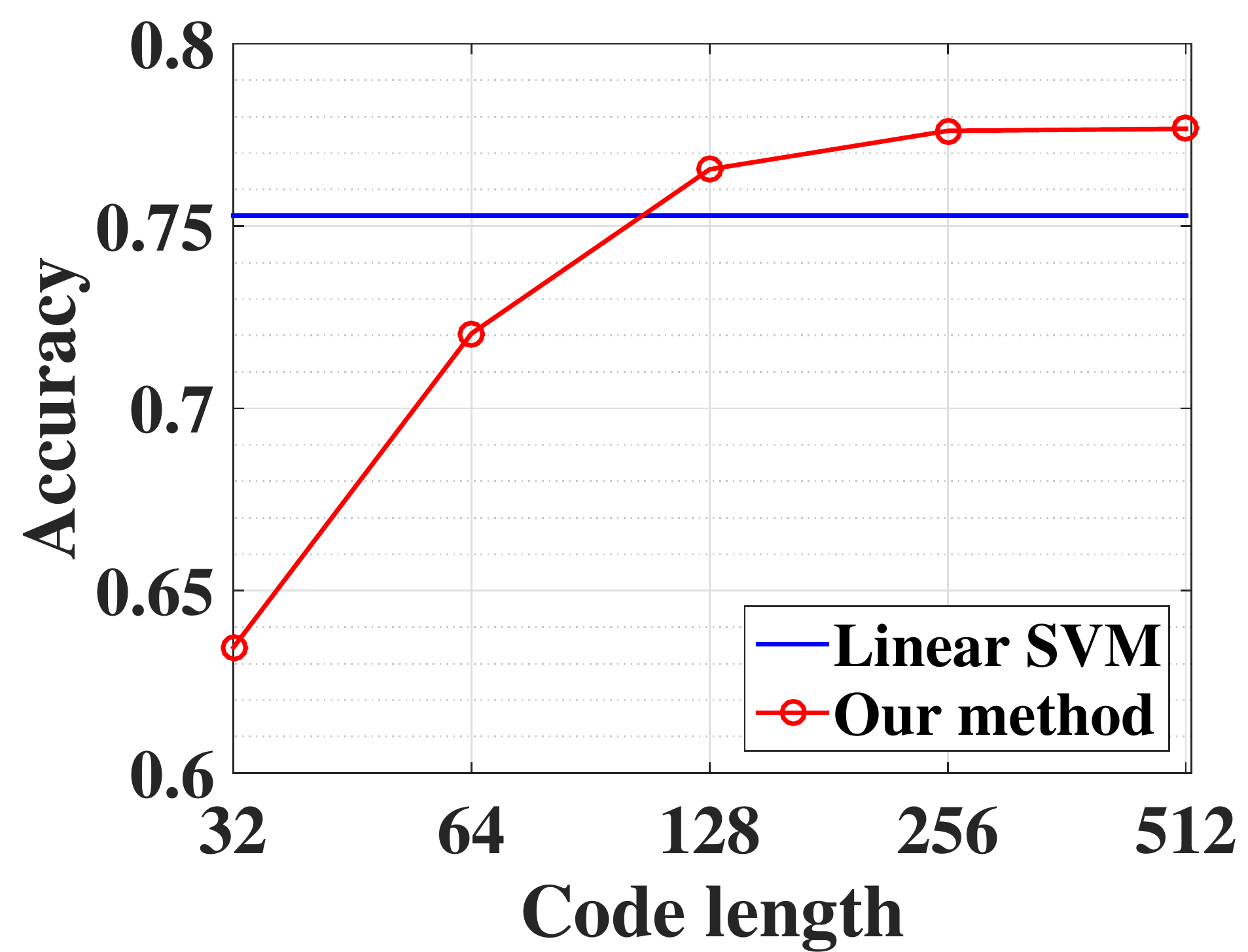}
\caption{Classification accuracy.}
\end{subfigure}
\begin{subfigure}[b]{0.23\textwidth}
\includegraphics[width=\textwidth]{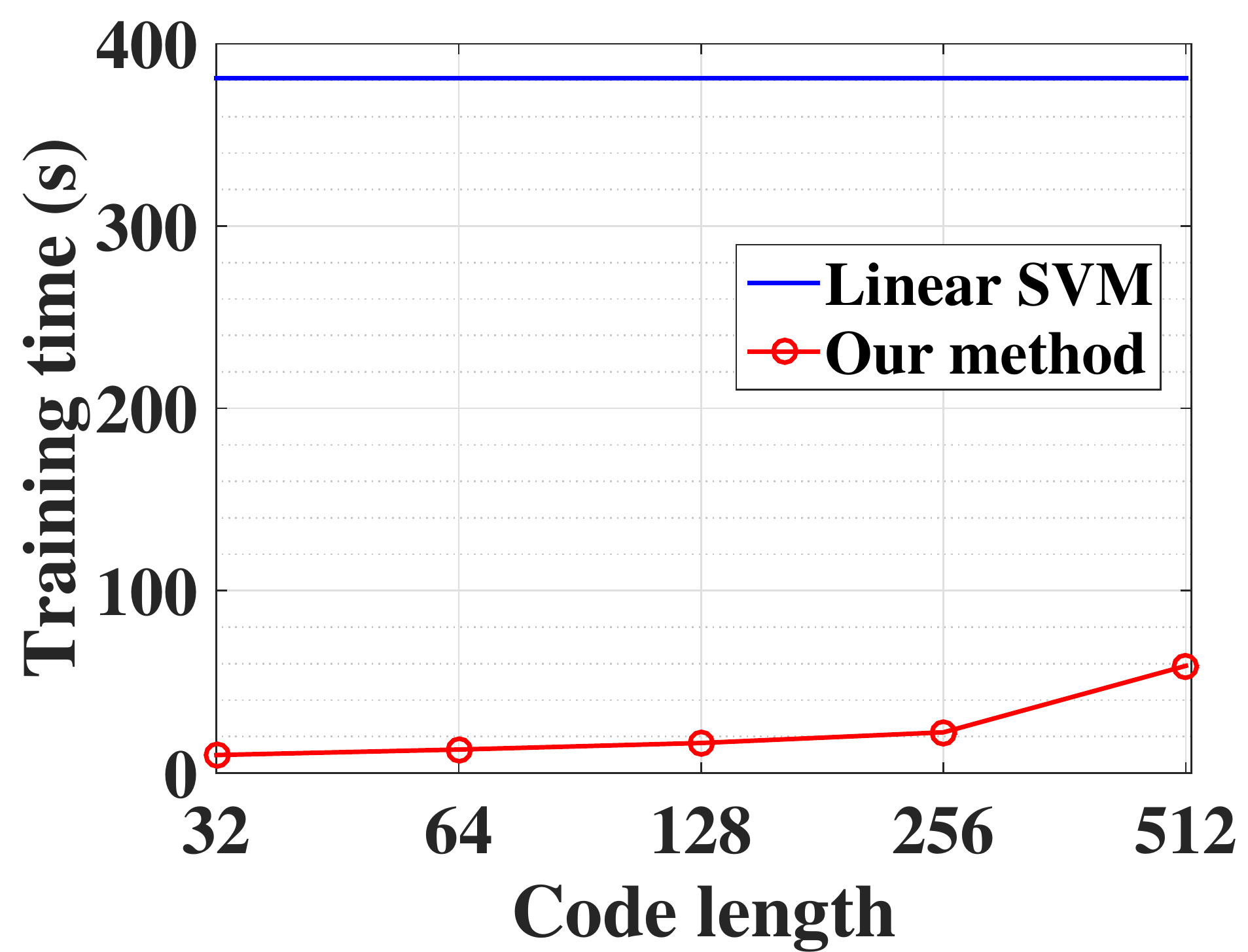}
\caption{Model training complexity.}
\end{subfigure}
\caption{\small Comparison of our method with the LibLinear implementation of Linear SVM for classification on the SUN dataset with 108K images from 397 scene categories. By coding both the image feature and learned classifiers with a small number of hash bits, our method achieves better results than Linear SVM, even with much smaller model training complexity. }
\label{fig:acc_sun}
\vspace{-0.2in}
\end{figure}

\section{Related work}
\label{Sec:related}

Let us first review the related works which strongly motivate ours. They can be roughly cast into two categories:

\vspace{0.1in}
\noindent \textbf{Hashing for fast image search and beyond}: learning compact hash codes~\cite{ICML13SHEN,SH08} recently becomes a hot research topic in the computer vision community. The proliferation of digital photography has made billion-scale image collections a reality, and efficiently searching a similar image to the query is a critical operation in such image collections. 

The seminal work of LSH~\cite{LSH99} sheds a light on fast image search with theoretic guarantee. In a typical pipeline of hash code based image search~\cite{SSH2012,xia2015sparse}, a set of hashing functions are generated either in an unsupervised manner or learned by perfectly separating similar / dissimilar data pairs on the training set. The latter is also called supervised hashing since it often judges the similarity of two samples according to their semantic labels. For an unseen image, it finds the most similar images in the database by efficiently comparing the corresponding hash codes. It can be accomplished in sub-linear time using hash buckets~\cite{LSH99}. Representative methods in this line include Binary Reconstructive
Embedding (BRE \cite{BRE2009}), Minimal Loss Hashing (MLH \cite{MLH2011}), Kernel-Based Supervised Hashing
(KSH \cite{KSH2012}), CCA based Iterative Quantization (CCA-ITQ \cite{gong2013iterative}), FastHash \cite{lin2014fast}, Graph Cuts Coding (GCC \cite{ge2014graph}), \etc.

The success of hashing technique is indeed beyond fast image search. For example, Dean et al.~\cite{Dean13} used hash tables to accelerate the dot-product convolutional kernel operator in large-scale object detection, achieving a speed-up of approximately 20,000 times when evaluating 100,000 deformable-part models.

\vspace{0.1in}
\noindent \textbf{Hashing for large-scale optimization}: Noting the Hamming distance is capable of faithfully preserve data similarity, it becomes a natural thought to extend it for approximating non-linear kernels. Optimizing with non-linear kernels generally require more space to store the entire kernel matrix, which prohibits its scalability to large data. Real vectors based explicit feature mapping~\cite{vedaldi2012efficient} partially remedies above issue, approximating kernel functions by the inner product between real vectors. However, they typically require high dimension towards an accurate approximation, and is thus beyond the scope of most practitioners. A more recent strand of research instead approximates non-linear kernel with binary bits, of which the prime examples can be found in~\cite{mu2014hash,li2013sign,li2011hashing}. In particular, Mu et al. developed a random subspace projection which transforms original data into compact hash bits. The inner product of hash code essentially plays the role of kernel functions. Consequently, the non-linear kernel SVM as their problem of interest can be converted into a linear SVM and resorts to efficient linear solvers like LibLinear~\cite{liblinear08}.

The philosophy underlying all aforementioned works can be summarized as binarizing the features and harnessing the compactness of binary code. We here argue that the potential of hashing technique in the context of large-scale classification has not been fully explored yet. Related research is still in its embryonic stage. For example, a recent work by Shen et al.~\cite{SDH15} proposed a Supervised Discrete Hashing (SDH) method under the assumption that good hash codes were optimal for linear classification. However, similar to other methods, SDH still classified the learned binary codes by real-valued weights. Thus the test efficiency for binary codes is still not improved compared to real-valued features.

After surveying related literature, we are motivated to advocate in this paper an extreme binary learning model, wherein both image features and classifiers shall be represented by binary bits. This way the learned models get rid of real-valued weight vectors and can fully benefit from high-optimized hash bit operators such as XOR.

\section{The proposed  model}
\label{Sec:prop}

Suppose that we have generated a set of binary codes $\BB = \{\Bb_i\}_{i=1}^n \in \{-1,1\}^{r \times n}$, where $\Bb_i$ is the
$r$-bit binary code for original data $\Bx_i$ from the training set $\BX = \{\Bx_i\}_{i=1}^n$. For simplicity, we assume a linear hash function
\begin{eqnarray}
h(\Bx) = \sgn(\BP^\top \Bx),
\end{eqnarray}
where $\BP \in \mathbb{R}^{d \times r}$.

In the context of linear classification, the binary codes  $\Bb$ is classified according to the maximum of the score vector
\begin{equation}
\BW^{\T}\Bb =  [\Bw_1^{\T}\Bb, \cdots,
\Bw_C^{\T}\Bb]^{\T},
\end{equation}
where  $\Bw_c \in \{-1,1\}^{r}$
is the binary parameter vector for class $c \in [1, \cdots, C]$.
Taking advantage of the binary nature of both $\Bw_c$ and $\Bb$, the inner product $\Bw_c^\top \Bb$ can be efficiently computed by $r - 2\mathbb{D_H}(\Bw_c,\Bb)$, where $\mathbb{D_H}(\cdot,\cdot)$ is the Hamming distance. Thereby the standard classification problem is transformed to searching the minimum from $C$ Hamming distances (or equivalently the maximum of binary code inner products).

Following above intuition, this paper proposes a multi-class classification framework, simultaneously learning the binary feature codes and classifier. Suppose $\Bb_i$ is the binary code of sample $\Bx_i$ and it shall be categorized as class $c_i$. Ideally, it expects the smallest Hamming distance (or largest inner product) to $\Bw_{c_i}$, in comparison with other classifier $\Bw_c$, $c \neq c_i$. An intuitive way of achieving this is through optimizing the inter-class ``margin". Formally, we can minimize the loss $\ell \big( -(\Bw_{c_i}^\top \Bb_i - \Bw_c^\top \Bb_i) \big)$, $\forall c$, where $\ell(\cdot)$ is a generic loss function. We re-formulate multi-class classification problem as below:
\begin{align}
\label{eqn:obj_general}
\min_{\BW, \BB} & \quad \sum_{i=1}^n \sum_{c=1}^C \ell \big( -(\Bw_{c_i}^\top \Bb_i - \Bw_c^\top \Bb_i) \big)  \\
\st & \quad \Bb_i \in \{-1,1\}^r, ~ \forall i, \;\Bw_c \in \{-1,1\}^r, ~  \forall c. \nonumber
\end{align}

We instantiate $\ell(\cdot)$ with the exponential loss (Section~\ref{Sec:exploss}) and hinge loss (Section~\ref{Sec:hingeloss}). In fact, the loss function in Problem~\eqref{eqn:obj_general} can be broadly defined. Any proper loss function $\ell(\cdot)$ can be applied as long as it is monotonically increasing.

\subsection{Learning with exponential loss}
\label{Sec:exploss}
Using the exponential loss function, we have the formulation below:
\begin{align}
\label{eqn:obj}
\min_{\BW, \BB} & \sum_{i=1}^n \sum_{c=1}^C \exp \left[ -(\Bw_{c_i}^\top \mathbf{b}_i - \Bw_c^\top \mathbf{\Bb}_i) \right]  \\
\st & \mathbf{\Bb}_i \in \{-1,1\}^r,~ \forall i,\;
 \Bw_c \in \{-1,1\}^r, ~\forall c. \nonumber
\end{align}

We tackle problem \eqref{eqn:obj} by alternatively solving the two sub-problems with $\BW$ and $\BB$, respectively.

\subsubsection{Classifying binary codes with binary Weights}
\label{SEC:W}

Assume $\BB$ is known. We iteratively update $\BW$ row by row, i.e., one bit each time for $\Bw_c, c = 1, \cdots, C$, while keep all other $r-1$ bits fixed.
Let $\Bw(k)$ denote the $k_{\mathrm{th}}$ entry of $\Bw$ and $\Bw(\setminus k)$ the vector which zeros its $k$-th element. We then have
\begin{align}
\label{EQ:decom}
\exp(\Bw^\top \Bb) =& \exp\left[ \Bw(\setminus k)^\top \Bb + \Bw(k)\Bb(k)\right]\nonumber\\
=& \exp\left[ \Bw(\setminus k)^\top \Bb \right] \cdot \exp\left[\Bw(k)\Bb(k)\right].
\end{align}
It can be verified that
\begin{eqnarray}
\label{EQ:linear1}
&& \exp\left[\Bw(k)\Bb(k)\right] \nonumber \\
&=& \left\{ \begin{array}{rc}
\frac{e^{-1}+e^1}{2} +  \frac{e^{-1}-e^1}{2} \cdot \Bw(k), & \Bb(k) = -1\\
\frac{e^{-1}+e^1}{2} - \frac{e^{-1}-e^1}{2} \cdot \Bw(k), &\Bb(k) = 1.
\end{array}
\right.
\end{eqnarray}

Denote $u = \frac{e^{-1}+e^1}{2}$ and $v = \frac{e^{-1}-e^1}{2}$. Equation \eqref{EQ:linear1} can be simplified by  the sign of $\Bb(k)$ as
\begin{equation}
\label{EQ:linear2}
\exp\left[\Bw(k)\Bb(k)\right] = u - v \cdot \Bb(k)\Bw(k).
\end{equation}
Equation \eqref{EQ:linear2} clearly shows the exponential function of the product of two binary bits equals to a linear function of the product. By applying \eqref{EQ:decom} and \eqref{EQ:linear2}, we write the loss term in \eqref{eqn:obj} as follows.
\begin{align}
\label{EQ:loss}
&\exp \left[ -(\Bw_{c_i}^\top \mathbf{b}_i - \Bw_c^\top \mathbf{\Bb}_i) \right] \\
=& \exp(\Bw_c^\top \mathbf{\Bb}_i) \cdot \exp(-\Bw_{c_i}^\top \mathbf{b}_i) \nonumber\\
=& \gamma_{ick} \cdot \left[u - v \cdot \Bb_i(k)\Bw_c(k))\right] \cdot \left[u + v \cdot \Bb_i(k)\Bw_{c_i}(k))\right],\nonumber
\end{align}
where the constant $\gamma_{ick} = \exp\left[ \big(\Bw_c(\setminus k) - \Bw_{c_i}(\setminus k)\big)^\top \Bb_i \right]$.
Clearly, the non-linear exponential loss becomes a quadratic polynomial with regard to $\Bw_c(k)$.

After merging terms with the same orders, optimizing problem~\eqref{eqn:obj}  with regard to $\Bw^k= [\Bw_1(k); \cdots; \Bw_C(k)]$ becomes
\begin{align}
\label{BQ}
\Bw^k \leftarrow & \arg \min_{\Bw^k} \; \frac{1}{2} \Bw^{k^\top } \BH^k \Bw^k +  \Bw^{k\top}\Bg^k,
\end{align}
where
\begin{equation}
\begin{array}{rcl}
\BH^k &=&  -2v^2 \BY \mathbf{\Gamma}^k, \\
\Bg^k &=& uv \BY (\Bb^k \odot \mathbf{\Gamma} \Bone) - uv \Gamma^\top \Bb^k.
\end{array} 
\end{equation}
Here $\mathbf{\Gamma}^{k} \in \mathbb{R}^{n \times C}$ includes its entries $\gamma_{ick}$, $\Bb^k$ is the $n$-dimension vector including the $k_{\mathrm{th}}$ binary bits of training data. $\BY \in \mathbb{R}^{C \times n}$ is the label matrix whose  entry $y_{ci}$ at coordinate $(c,i)$ equals to 1 if sample $\Bx_i$ belongs to class $c$ and 0 otherwise. $\odot$ denotes the element-wise product. $\Bone$ is the vector with all ones.

Problem \eqref{BQ} is a binary quadratic program (BQP) and can be efficiently solved by sequential bit flipping operation. A local optimum can be guaranteed.
We solve problem \eqref{BQ} by investigating the local optimality condition. Intuitively, for any local optimum, flipping any of its hash bits will not decrease the objective value of problem \eqref{BQ}. 
Let $\BH_{*,c}, \BH_{c,*}$ denote the column or row vector indexed by $c$ respectively. $\Bg(c)$ and $\Bw(c)$ represents the $c_{\mathrm{th}}$ element of $\Bg$ and $\Bw$, respectively. In problem \eqref{BQ}, collecting all terms pertaining to $\Bw(c)$ obtains
\begin{align}
f(\Bw^k(c)) &= \frac{1}{2} \left( \BH^{k\top}_{*,c} + \BH^{k}_{c,*} \right) \Bw^k \cdot \Bw^k(c) + \Bg^k(c) \cdot \Bw^k(c). \nonumber 
\end{align}

By flipping $\Bw(c)$, the gain of the objective function of problem \eqref{BQ} is
\begin{align}
\Delta_{\Bw^k(c)\rightarrow -\Bw^k(c)} = f(-\Bw^k(c)) - f(\Bw^k(c)).
\end{align}

Regarding the local optimality condition of problem \eqref{BQ}, we have the observation below:
\begin{proposition}
\label{thm:localopt}
\textbf{(Local optimality condition)}:
Let $\Bw^\ast$ be a solution of problem \eqref{BQ}. $\Bw^\ast$ is a local optimum when the condition $\Delta_{\Bw(c)\rightarrow -\Bw(c)} \ge 0$ holds for $c=1,\ldots,C$.
\end{proposition}

\begin{proof}
The conclusion holds by a simple application of proof of contradiction. Recall that $\Bw$ is a binary vector. Flipping any bit of $\Bw$ will incur specific change of the objective function as described by $\Delta_{\Bw(i)\rightarrow -\Bw(i)}$. When the changes incurred by all these flipping operations are all non-negative, $\Bw^\ast$ is supposed to be locally optimal. Otherwise, we can flip the bit with negative $\Delta_{\Bw(i)\rightarrow -\Bw(i)}$ to further optimize the objective function.
\end{proof}

With the above analysis, we summarize our algorithm for updating the $C$-bit $\Bw^k$ as in Algorithm \ref{alg:bitflip}.

\begin{algorithm}[tb]
   \caption{Sequential bit flipping}
   \label{alg:bitflip}
   \begin{small}   
\begin{algorithmic}[1]
   \WHILE{local optimality condition does not hold}
    \STATE Calculate the bit-flipping gain $\Delta_{\Bw(c)\rightarrow -\w(c)}$ for $c=1,\ldots,C$;
    \STATE Select $\hat c = \arg \min_c \Delta_{\Bw(c)\rightarrow -\Bw(c)}$ and $\Delta_{min} = \min_c \Delta_{\Bw(k)\rightarrow -\Bw(c)}$;
    \IF{$\Delta_{min} < 0$} 
    \STATE Set $\Bw(c) \leftarrow -\Bw(c)$;
    \ELSE
    \STATE Exit;
    \ENDIF
   \ENDWHILE
\end{algorithmic}
\end{small}
\end{algorithm}


\subsubsection{Binary code learning}
Similar as the optimization procedure for $\BW$, we solve for $\BB$ by a coordinate descent scheme. In particular, at each iteration, all the rest $r-1$ hash bits are fixed except for the $k$-th hash bit $\Bb^k = \left[\Bb_1(k); \cdots; \Bb_n(k)\right]$. Let $\Bb_i(\setminus k)$ denote the vector which zeros its $k$-th element $\Bb_i(k)$.
We rewrite equation \eqref{EQ:loss} w.r.t $\Bb_i(k)$ as 
\begin{align}
& \exp \left[ -(\w_{c_i}^\top \mathbf{b}_i - \w_c^\top \mathbf{b}_i)\right] \\
=& z_{ick}  \cdot \exp \left[  ( \w_c(k) - \w_{c_i}(k)) \mathbf{b}_i(k) \right] \nonumber 
\end{align}
where $z_{ick}  = \exp \left[ (\w_c - \w_{c_i})^\top  \mathbf{b}_i(\setminus k) \right]  $.

Similar as \eqref{EQ:linear1}, we have
\begin{eqnarray}
&& \exp \left[ ( \w_c(k) - \w_{c_i}(k)) \mathbf{b}_i(k) \right] \nonumber \\
&=& \left\{ \begin{array}{rc}
 0, & \w_{c_i}(k) = \w_c(k) \\
\frac{e^{-2}+e^2}{2} +  \frac{e^{-2}-e^2}{2} \cdot \mathbf{b}_i(k), & \w_{c_i}(k) = 1,\w_c(k)=-1 \\
\frac{e^{-2}+e^2}{2} -  \frac{e^{-2}-e^2}{2} \cdot \mathbf{b}_i(k), & \w_{c_i}(k) = -1,\w_c(k)=1.
\end{array}
\right. \nonumber
\end{eqnarray}
We can see that, the non-linear exponential loss term becomes either a constant or linear function with regard to $\mathbf{b}_i(k)$.

Denote $u' = \frac{e^{-2}+e^2}{2}$, $v' = \frac{e^{-2}-e^2}{2}$. Let matrix $\mathbf{Z}^k \in \mathbb{R}^{n \times C}$ include its entry at the coordinate $(i,c)$ as $z_{ick}$ if $\w_{c_i}(k) = 1,\w_c(k)=-1$ and 0 otherwise; similarly let matrix $\mathbf{\bar{Z}}^k \in \mathbb{R}^{n \times C}$ include its entry at the coordinate $(i,c)$ as $z_{ick}$ if $\w_{c_i}(k) = -1,\w_c(k)=1$ and 0 otherwise. Then the loss  in \eqref{eqn:obj} can be written as w.r.t $\Bb^k$
\begin{align}
\label{P} 
&\sum_{i=1}^n \sum_{c=1}^C \exp \left[ -(\Bw_{c_i}^\top \mathbf{b}_i - \Bw_c^\top \mathbf{\Bb}_i) \right] \\
=&\sum_{i=1}^n \sum_{c=1}^C \mathbf{Z}^k (i,c) \cdot (u'+v'\Bb_i(k)) + \mathbf{\bar{Z}}^k (i,c) \cdot (u'-v'\Bb_i(k))\nonumber\\
=& \sum_{i=1}^n \mathbf{z}^k(i) \cdot  (u'+v'\Bb_i(k)) + \mathbf{\bar{z}}^k(i)\cdot(u'-v'\Bb_i(k)),
\end{align}
where $\Bz^k(i) = \sum_{c=1}^C \BZ^k (i,c)$ and
$\bar{\Bz}^k(i) = \sum_{c=1}^C \bar{\BZ}^k (i,c)$.

Then we have the following optimization problem
\begin{align}
\min_{\Bb^k} & \; v' (\Bz^k - \bar{\Bz}^k)^\top \Bb^k \nonumber\\
\st & \; \Bb^k \in \{-1,1\}^r,
\end{align}
which has a optimal solution $\Bb^k = -\sgn(v' (\Bz^k - \bar{\Bz}^k))$.

\begin{figure}
\centering
\begin{subfigure}[]{0.23\textwidth}
\includegraphics[width=1\textwidth]{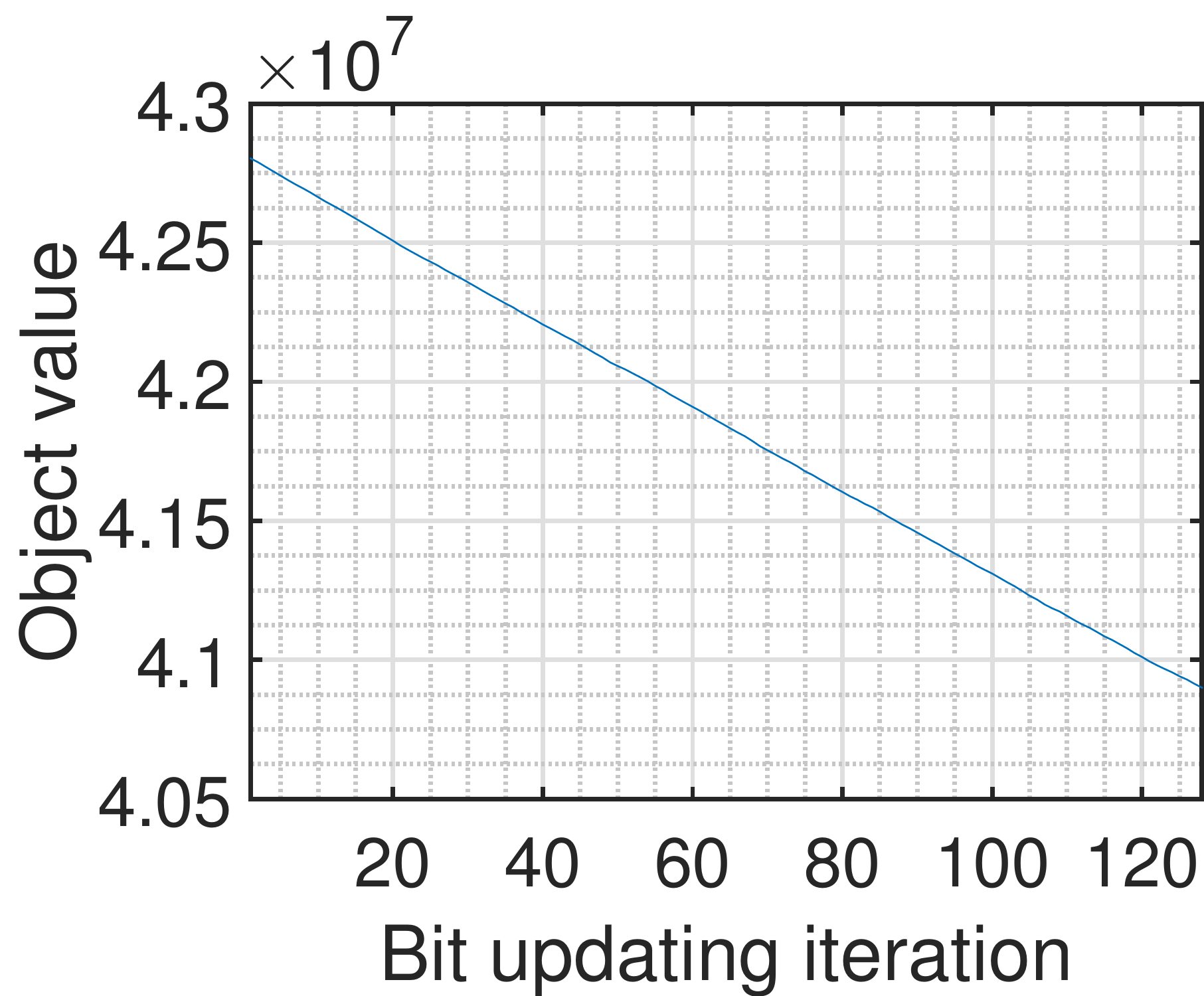}
\caption{Updating $\BW$.}
\end{subfigure}
\begin{subfigure}[]{0.23\textwidth}
\includegraphics[width=1\textwidth]{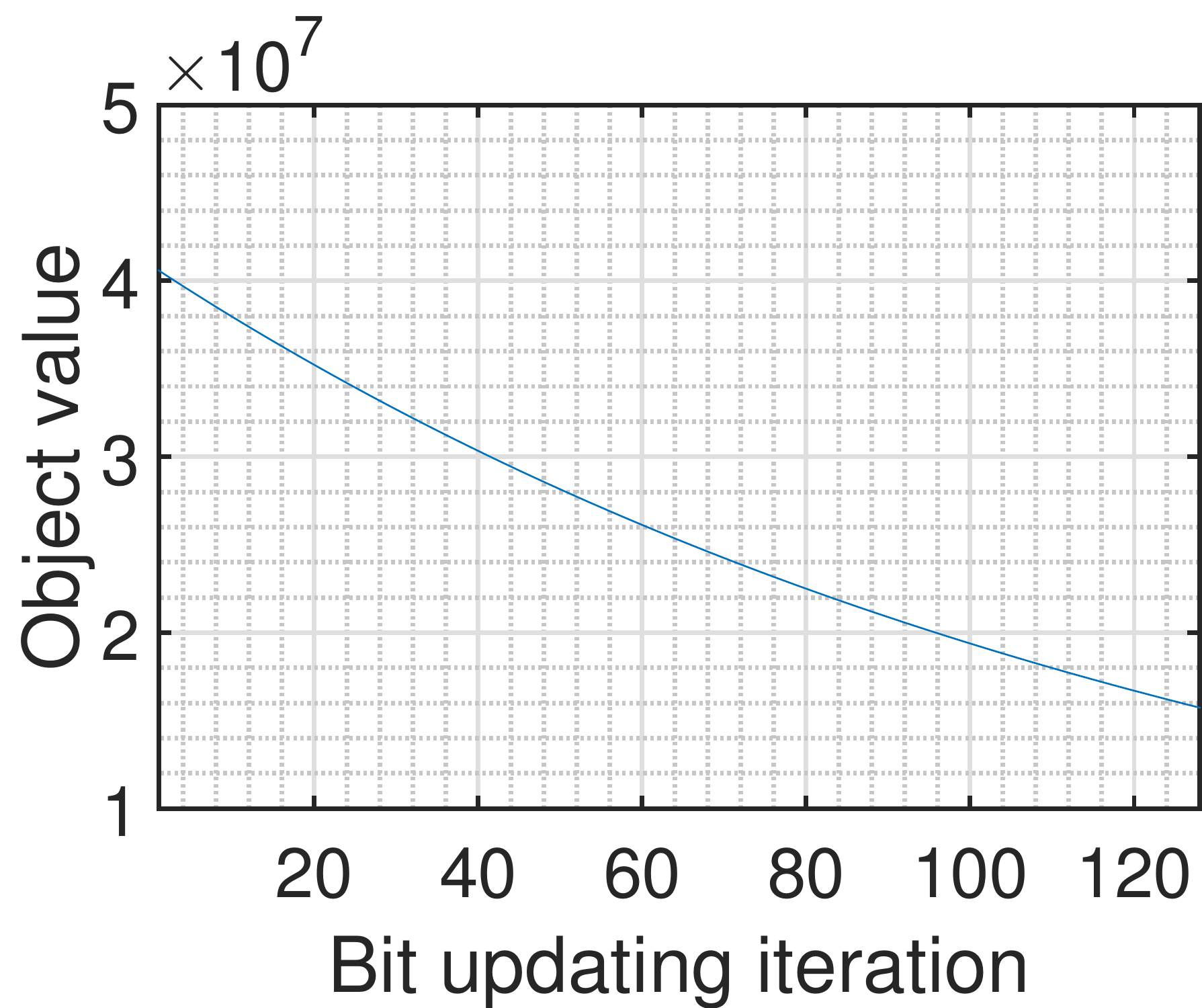}
\caption{Updating $\BB$.}
\end{subfigure}
\caption{Object value as a function of bit updating iteration $k$ in optimizing $\BW$ and $\BB$ on the dataset of SUN397.}
\label{fig:obj}
\end{figure}
Figure~\ref{fig:obj} shows the objective value as a function of the bit updating iteration number. As can be seen, with the proposed coordinate descent optimizing procedure for both $\BW$ and $\BB$, the object value consistently decreases as updating the hash bits in each sub-problem. The optimization for the original problem \eqref{eqn:obj}  typically converges in less than 3 iterations of alternatively optimizing $\BW$ and $\BB$.

\subsection{Learning with the hinge loss}
\label{Sec:hingeloss}
With the widely used hinge loss, we obtain
\begin{align}
\label{EQ:svm}
\min_{\BB,\BW,\mathbb{\xi}} &\quad \lambda ||\BW||^2 + \sum_{i=1}^n \sum_{c=1}^C \xi_{i} \\
\st & \quad \forall i,c \quad  \Bw_{c_i}^{\T}\Bb_i + y_{ci} -  \Bw_c^{\T}\Bb_i \geq 1 - \xi_{i},\nonumber\\
&\quad  \Bb_i = h(\Bx_i), ~ \forall i, \;\w_c \in \{-1,1\}^r, ~  \forall c. \nonumber
\end{align}
Here $\xi_{i} \geq 0$ is the slack variable; 
$\BI_r$ is the $r \times r$ identity matrix. 
We denote 
$y_{ci} = 1$ if $\Bx_i$ belongs to class $c$ and $-1$
otherwise.
The non-negative constraint $\xi_{i} \ge 0$ always holds: If $\xi_{i} < 0$ for a particular $\Bb_i$, the first constraint in \eqref{EQ:svm} will be violated when $y_{ci} = 1$. In other words, the constraint corresponding to the case the label of $\Bb_i$ equals to $c$ ensures the non-negativeness of $\xi_{i}$.
This formulation is similar to the multi-class SVM \cite{crammer2002algorithmic} except that it is 
exposed to the binary constraints of both input features and classification matrix.

Similar as with the exponential loss, we tackle problem \eqref{EQ:svm} by alternatively solving the two sub-problems regarding to $\BW$ and $\BB$, respectively.
Observing  $||\BW||^2$ is constant, we can write problem \eqref{EQ:svm} as w.r.t. $\BW$, 

\begin{align}
\label{EQ:W1}
\min_{\BW} &\quad   \sum_{i=1}^n \sum_{c=1}^C  ( \Bw_c^{\T}\Bb_i -\Bw_{c_i}^{\T}\Bb_i)\\
\st &\quad  \Bw_c \in \{-1,1\}^r, \forall c. \nonumber
\end{align}

Collecting all terms with $\Bw_c,\;\forall c$, problem \eqref{EQ:W1} writes
\begin{align}
\label{EQ:Hinge3}
\min_{\BW} &\quad \sum_{c=1}^C \Bw_c^{\T}\big(\sum_{i=1}^n\Bb_i - C\sum_{i=1, c_i=c}^n\Bb_i\big) \\
\st & \quad \Bw_c \in \{-1,1\}^r, \forall c, \nonumber
\end{align}

which has optimal solution
\begin{align}
\label{EQ:W}
\Bw_c &=  \sgn\big( C \sum_{i=1, c_i=c}^n\Bb_i - \sum_{i=1}^n\Bb_i \big), \forall c.
\end{align}

For the sub-problem regarding to $\BB$, we first let matrix $\BW^o$ of size $r \times n$ include its $i_{\mathrm{th}}$ column as $\Bw^o_i =  \sum_{c=1}^C \Bw_c -  C\Bw_{c_i}$. Problem \eqref{EQ:W1} writes w.r.t. $\BB$ as
\begin{align}
\label{EQ:B1}
\min_{\BB}  &\;  \trace(\BW^{o\T}\BB)\\
\st & \;  \BB \in \{-1,1\}^{r \times n}. \nonumber
\end{align}
$\BB$ can be efficiently computed by $-\sgn(\BW^{o})$. It is clear that, both of the  two sub-problems associated with $\BB$ and $\BW$ have closed-form solutions, which significantly reduce the computation overhead of classifier training and binary code learning.

\subsection{Binary code prediction}
With the  binary codes $\BB$ for training data $\BX$  obtained,  the hash function $h(\Bx) = \sgn(\BP^{\T}\Bx)$ is obtained by solving a simple linear regression system $$\BP = (\BX^{\T}\BX)^{-1}\BX^{\T}\BB.$$
Then for a new sample $\Bx$, the classification is conducted by searching the minimum of $C$ Hamming distances:
\begin{equation}
c^{\ast} = \arg \min_c \{\mathbb{D_H}\big(\Bw_c, h(\Bx)\big)\},\; c = 1, \cdots, C.
\end{equation}

The binary coding step occupies the main computation in the testing stage, which is $O(dr)$ in time complexity.

\section{Experiments}
\label{Sec:exp}

\begin{table*}
\def\arraystretch{1.4}
\centering
\caption{Comparative results in terms of test accuracy (\%), training and testing time (seconds). Experiments are conducted on a standard PC with a quad-core Intel CPU and 32GB RAM. For LSH, CCA-ITQ, SDH and our methods, 128 bits are used. OVA-SVM and Multi-SVM is performed with the implementation of LibLinear, where the best accuracies are reported  with parameter $c$ chosen from \{1e-3, 1e-2, 1e-1, 1, 1e1, 1e2, 1e3\}.}
\begin{tabular}{c|ccc|ccc|ccc}
\hline\hline
\multirow{2}{*}{ Method} & \multicolumn{3}{c|}{SUN397} & \multicolumn{3}{c|}{ImageNet} & \multicolumn{3}{c}{CIFAR-10}\\
& acc (\%) & train time  & test time & acc (\%) & train time  & test time& acc (\%) & train time  & test time\\
\hline
OVA-SVM &\textbf{77.39}& 818.87&1.55e-5 &\textbf{79.84} &151.02 &1.15e-5 &57.1 &55.17&4.02e-7\\
Multi-SVM &75.28  &380.94 &1.01e-5 &79.48 & 93.12& 1.21e-5&{57.7} &35.55&4.27e-7\\
\hline
LSH &54.11 & 417.42 &7.75e-6 &58.16 &107.41 &1.32e-5 &39.0 &39.64 &2.26e-6 \\
CCA-ITQ &69.33 & 452.34&8.78e-6 &76.30 &142.95 & 1.25e-5&56.4 &47.23& 2.35e-6\\
SDH  &72.56 & 2522.33&7.43e-6 &76.64 &1102.21 & 1.43e-5&55.3&115.63 &2.32e-6\\
Ours-Exponential& 75.44&772.11 &3.67e-6&\textbf{79.04} &245.14 & 6.54e-6&\textbf{59.2} &16.01 &2.01e-6\\
Ours-Hinge&\textbf{76.56} &\textbf{16.45} &3.86e-6  &77.88 & \textbf{35.16}& 6.86e-6&54.2 & \textbf{2.31}&2.13e-6\\
\hline\hline
\end{tabular}
\label{tab:3data}
\end{table*}

In this section, we evaluate the proposed two methods on three large-scale datasets: CIFAR-10\footnote{\url{http://www.cs.toronto.edu/~kriz/cifar.html}.},  SUN397 \cite{xiao2010sun} and ImageNet \cite{deng2009imagenet}. As a subset of the well-known 80M tiny image collection \cite{80Mtiny2008}, the CIFAR-10 dataset consists of 60,000 images which are manually labelled as 10 classes with $6,000$ samples for each class. We represent each image in this dataset by a GIST feature vector \cite{GIST2001} of dimension $512$. The whole dataset is split into a test set with $1,000$ samples and a training set with all remaining samples.
SUN397 \cite{xiao2010sun} contains about 108K images
from 397 scene categories, where each image is represented by a 1,600-dimensional feature vector extracted by PCA from 12,288-dimensional Deep Convolutional Activation Features \cite{gong2014multi}. We use a subset of this dataset including 42 categories  with each containing more than 500 images; 100 images are sampled uniformly randomly from each category to form a test set of 4,200 images.
As a subset of ImageNet \cite{deng2009imagenet}, the large dataset ILSVRC 2012 contains over 1.2 million images of totally 1,000 categories. We form the evaluation database by the 100 largest classes with total 128K images from training set, and 50,000 images from validation set as test set. We use the 4096-dimensional features extracted by the convolution neural networks (CNN) model \cite{krizhevsky2012imagenet}.  

The proposed methods (denoted by Ours-Exponential and Ours-Hinge for the exponential and hinge loss, respectively) are extensively  compared with two popular linear classifiers: one-vs-all linear  SVM (OVA-SVM) and multi-class SVM (Multi-SVM \cite{crammer2002algorithmic}), both of which are implemented using the LibLinear software package \cite{liblinear08}. For these two methods, we tune the parameter $c$ from  the range [1e-3, 1e3] and the best results are reported.  
 We also compare our methods  against several state-of-the-art binary code learning methods including Locality Sensitive Hashing (LSH) implemented by signed random projections, CCA-ITQ \cite{gong2013iterative}, and SDH \cite{SDH15}. 
The classification results of these hashing methods are obtained by performing the multi-class linear SVM over the predicted binary codes by the corresponding hash functions. 
We use the public codes and suggested parameters of these methods from the authors. We extensively evaluate these compared methods in terms of storage memory overhead, classification accuracy and computation time.

\subsection{Accuracy and computational efficiency}
In this part, we extensively evaluate the proposed two methods with the compared algorithms in both classification accuracy and computation time.  We use 128-bit for our method and the three hashing algorithms  LSH, CCA-ITQ and SDH. 

We report   the results on {SUN397}, {ImageNet} and {CIFAR-10} in Table~\ref{tab:3data}. Between the  algorithms with different loss functions of our proposed model, we can see that the method with the exponential loss performs slightly better than that with the hinge loss, while the latter one  benefits from a much more efficient training. This is not surprising because Ours-Hinge solves two alternating sub-problems  both with closed-form solutions.

 Compared to other methods, the results clearly show that,  our methods  achieve competitive (or even better) classification accuracies on all three large-scale datasets  with the state-of-the-art linear SVMs. In the meanwhile,  even being constrained to learn binary classification weights, our methods obtain much better results than the compared hashing algorithms. Specifically, Ours-Exponential outperforms the best results obtained the hashing algorithms by 2.88\%,  2.4\% and 2.8\% on SUN397, ImageNet and CIFAR-10, respectively.

In terms of training time, we can see that  our method with the hinge loss runs way faster than all other methods on all the evaluated three datasets. In particular, on SUN397, Ours-Hinge is $50\times$ and $23\times$ faster than the LibLinear implementations of one-vs-all SVM and mult-class SVM. Compared with the hashing algorithm, our method runs over $28\times$ faster than the fastest LSH followed by Liblinear. For testing time, the benefit of binary dimension reduction for our methods together with other three hashing algorithms is clearly shown on the SUN397 dataset with a large number of categories. Our methods requires less testing time than the hashing based classification  methods, which is  due to the extremely fast classification implemented by searching in the Hamming space.

We also evaluate the compared methods with different code lengths, where 
the detailed results on SUN397 and ImageNet  are shown in Figure~\ref{fig:acc_two}. From Figure~\ref{fig:acc_two}, it is clear that with a relatively small number of bits (\eg, 256 bits), our method can achieve close classification performance to or even better than the real-valued linear SVM with real-valued features. 
We can also observe that
our method consistently outperforms other hashing algorithms by noticeable gaps at all code lengths on SUN397. On the ImageNet dataset, our method achieves  marginally better results than SDH and CCA-ITQ, while much better than the random LSH algorithm.

\begin{figure}
\centering
\begin{subfigure}[]{0.3\textwidth}
\includegraphics[width=1\textwidth]{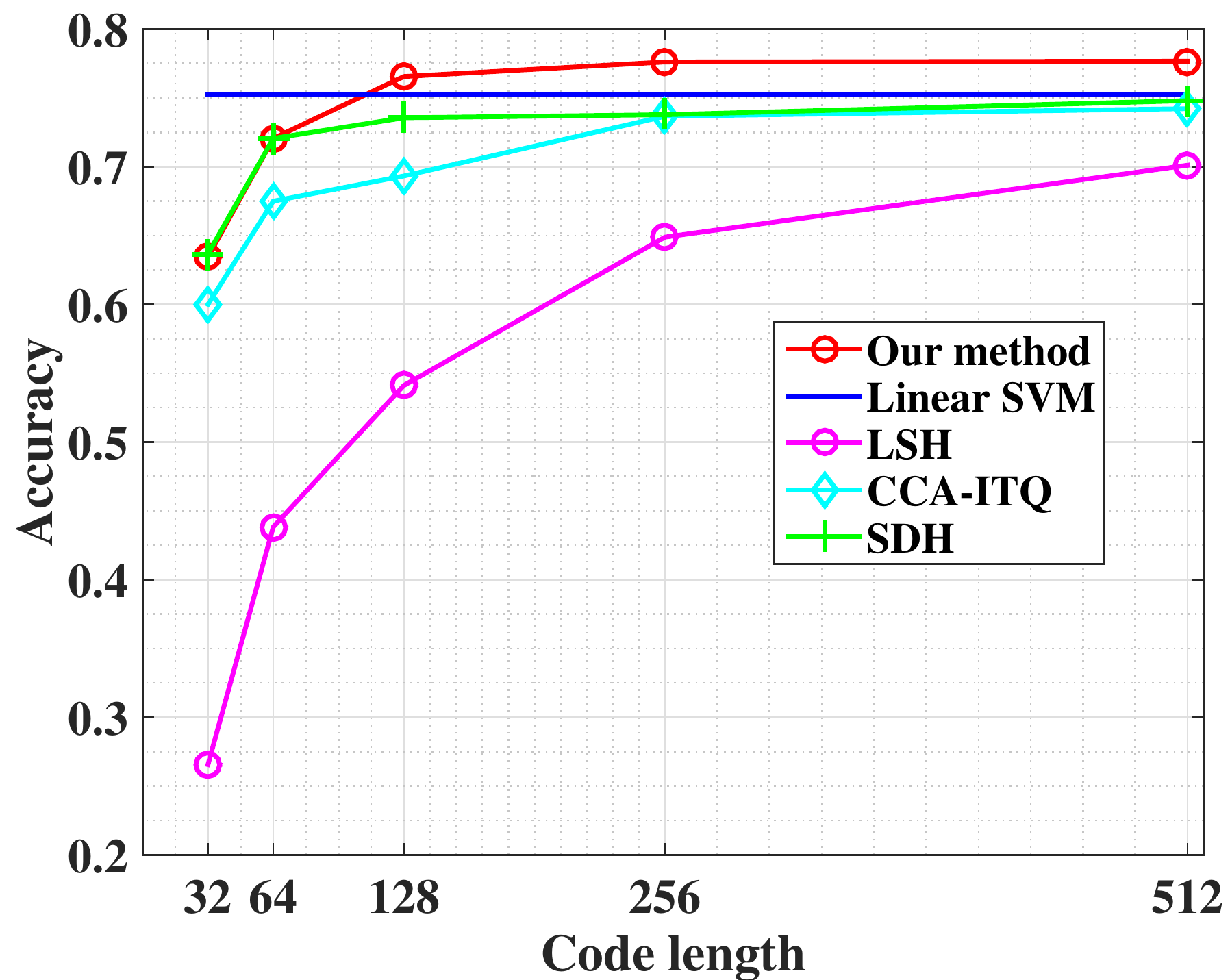}
\caption{ SUN397}
\end{subfigure}
\begin{subfigure}[]{0.3\textwidth}
\includegraphics[width=1\textwidth]{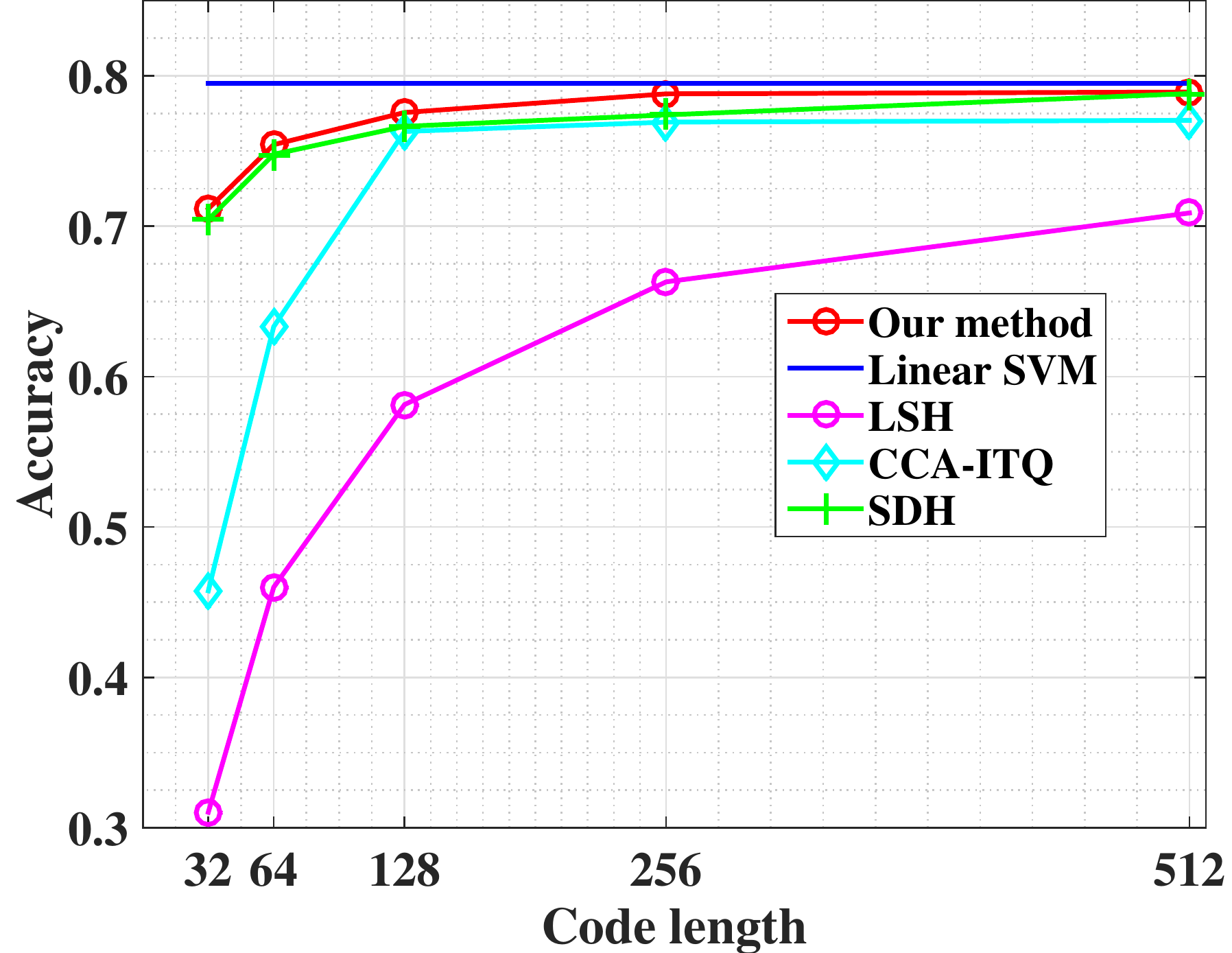}
\caption{ ImageNet}
\end{subfigure}
\caption{Comparative results of various algorithms in classification accuracy on SUN397 and ImageNet.}
\label{fig:acc_two}
\end{figure}

\begin{figure}
\centering
\begin{subfigure}[]{0.3\textwidth}
\includegraphics[width=1\textwidth]{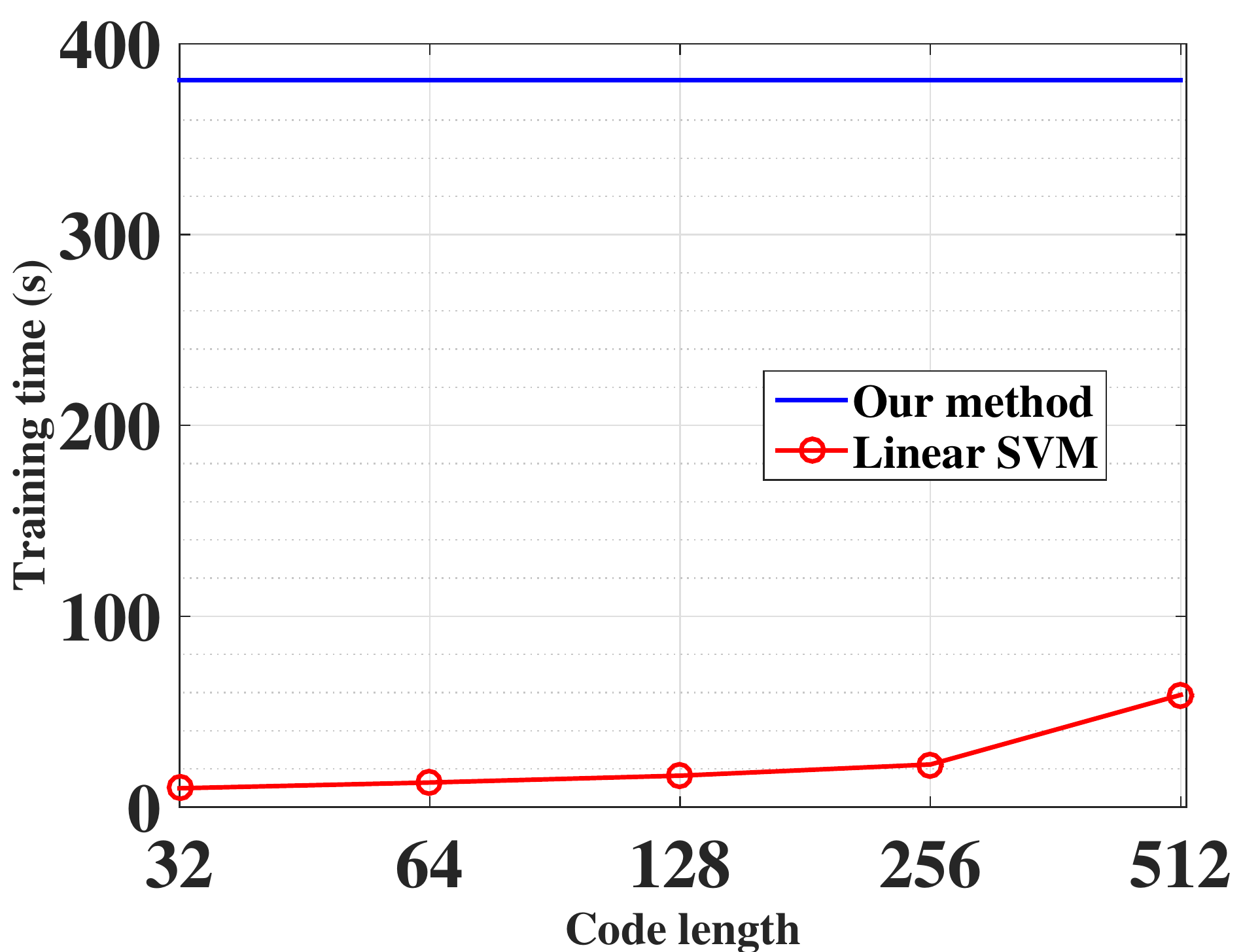}
\caption{ SUN397}
\end{subfigure}
\begin{subfigure}[]{0.3\textwidth}
\includegraphics[width=1\textwidth]{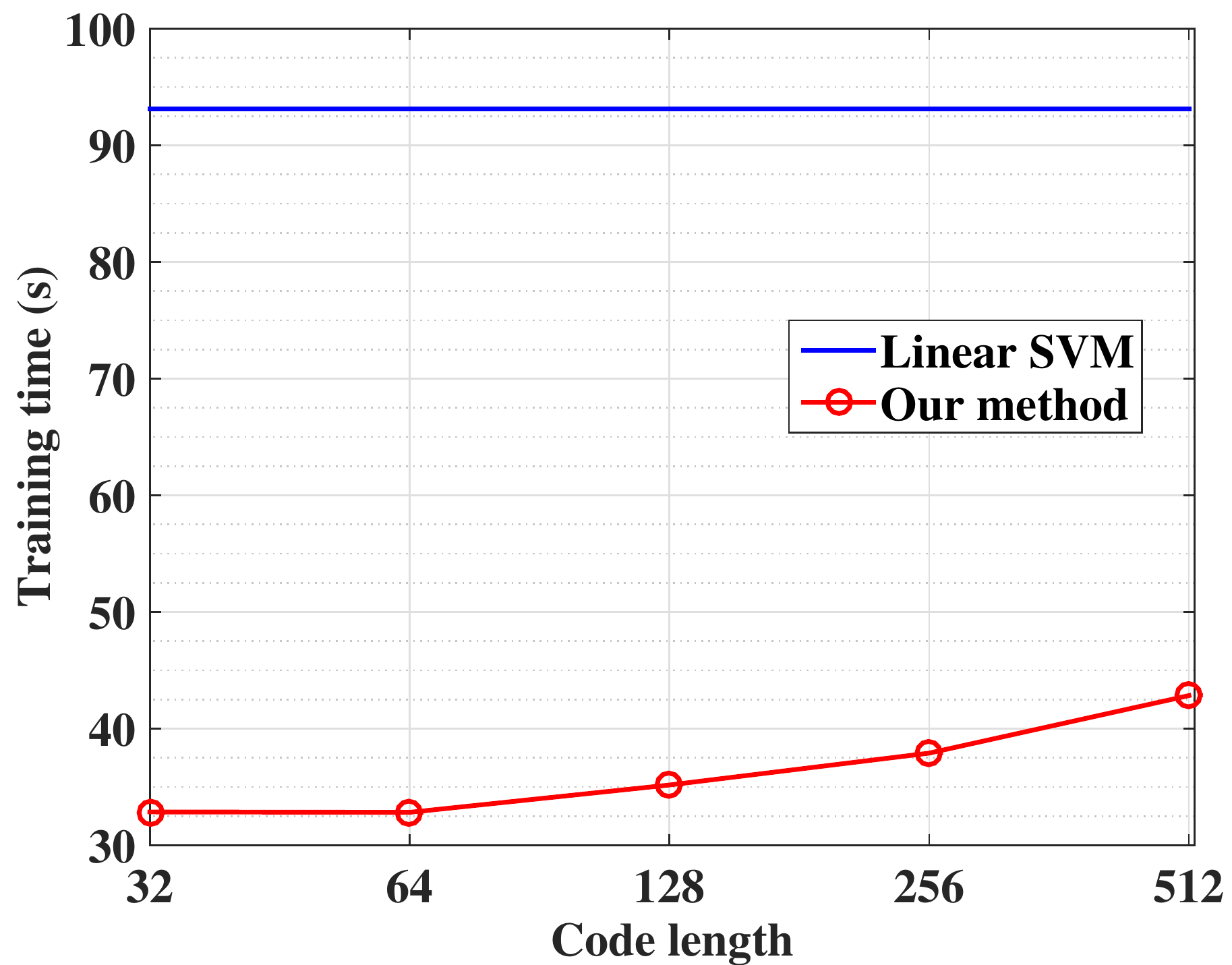}
\caption{ ImageNet}
\end{subfigure}
\caption{Comparative results of various algorithms in training time on SUN397 and ImageNet.}
\label{fig:time_two}
\end{figure}

Figure~\ref{fig:time_two} demonstrates the consumed training time by our method and Linear SVM by the Liblinear solver on two large-scale datasets. The  computation efficiency advantage of our method  is clearly shown. Our method has a nearly linear training time complexity with the code length, which can  facilitate its potentially applications in high dimensional binary vector learning.

\begin{table}[]
\def\arraystretch{1.4}
\centering
\caption{Memory overhead (MB) to store the training features and classification model using Linear SVM and our method (128-bit). Note that, for our method, the trained model includes the real-valued hashing matrix ($\BP$) and the binary classification weight matrix $\BW$.}
\begin{tabular}{c|cc|cc}
\hline\hline
\multirow{2}{*}{Dataset} & \multicolumn{2}{c|}{ Training features} & \multicolumn{2}{c}{Classification model}\\
& Linear SVM & Ours& Linear SVM & Ours\\
\hline
 ImageNet&3943.91 & 24.37&30.86 &9.92\\
SUN397 &1283.83 & 2.01&4.79 &1.54\\
\hline\hline
\end{tabular}
\label{tab:mem}
\end{table}
\subsection{Running memory overhead}
In this subsection, we compare our method with  linear SVM in term of running memory overhead for loading training features in the training stage and storing classification models in the testing stage. For our method, the trained model includes the real-valued hash function matrix ($\BP$) and the binary classification weight matrix.
 The results are reported in Table~\ref{tab:mem}. It is clearly shown that our approach requires much less memory than Linear SVM for loading both the training features and classification models. Taking the ImageNet database for example, Linear SVM needs over 150 times more RAM than our method for the training data, and over 3 times more RAM for the trained models.

\section{Conclusions}
\label{Sec:conc}

This work proposed a novel  classification framework, by which classification  was equivalently transformed to searching the nearest binary weight code in the Hamming space. Different from previous methods, both the feature  and classifier weight vectors were simultaneously learned with binary hash codes. Our framework could employ a large family of empirical loss functions, and we here especially studied the representative exponential and hinge loss. For the  two sub-problems regarding to the binary classifier and image hash codes, a binary quadratic program (BQP)  and linear program was formulated, respectively. In particular, for the BQP problem, a novel bit-flipping procedure which enjoys high efficacy and local optimality guarantee was presented. 
The two methods with exponential loss  and  hinge loss 
were extensively evaluated on several large-scale image datasets. Significant computation overhead reduction of model training and deployment were obtained, while without sacrifice of classification accuracies.

{
\bibliographystyle{IEEEtran}
\bibliography{FSRef}
}

\end{document}